\documentclass[wcp]{jmlr}

\usepackage{longtable}

\usepackage{booktabs}
\usepackage{times}
\usepackage{latexsym}

\usepackage[T1]{fontenc}

\usepackage{hyperref}
\usepackage[disable]{todonotes}

\usepackage{latexsym}
\usepackage{algorithm,algpseudocode}

\usepackage{microtype}
\usepackage{multirow}

\usepackage{svg}

\usepackage{esvect}
\usepackage{booktabs}
\usepackage{bbm}
\usepackage{amsfonts}       
\usepackage{nicefrac}       
\usepackage{microtype}      
\usepackage{hyphenat}
\usepackage{xcolor}
\usepackage{rotating}
\usepackage{makecell}
\usepackage{mathrsfs}

\usepackage{ragged2e}
\usepackage{lipsum}
\renewcommand{\thesubfigure}{\Roman{subfigure}}
\usepackage{mathtools}
\usepackage{arydshln}

\usepackage{savesym}
\savesymbol{Email}
\usepackage{marvosym}
\restoresymbol{MARV}{Email}

\usepackage[utf8]{inputenc}
\usepackage[T1]{fontenc}
\usepackage{hyphenat}
\usepackage{xspace}
\usepackage{amsmath}
\usepackage{amsfonts}
\usepackage{hyperref}
\usepackage{url}
\usepackage{booktabs}
\usepackage{multirow}
\usepackage{makecell}
\usepackage{caption}
\usepackage{minibox}
\usepackage{bbm}
\usepackage{bm}
\usepackage{graphicx}
\usepackage{balance}
\usepackage{mathtools}
\usepackage{color}
\usepackage{marvosym}
\usepackage{ifthen}
\usepackage{textcomp}
\usepackage{enumitem}
\usepackage{verbatim}
\usepackage{numprint}
\usepackage{balance}

\newcommand{\chatoDisplayMode}[1]{#1}

\definecolor{MyRed}{rgb}{0.6,0.0,0.0} 
\definecolor{MyBlack}{rgb}{0.1,0.1,0.1} 
\newcommand{\inred}[1]{{\color{MyRed}\sf\textbf{\textsc{#1}}}}
\newcommand{\frameit}[2]{
  \begin{center}
  {\color{MyRed}
  \framebox[.9\columnwidth][l]{
    \begin{minipage}{.85\columnwidth}
    \inred{#1}: {\sf\color{MyBlack}#2}
    \end{minipage}
  }\\
  }
  \end{center}
}

\newcommand{\note}[2][]{\chatoDisplayMode{\def\@tmpsig{#1}\frameit{{\Pointinghand} Note}{#2\ifx \@tmpsig \@empty \else \mbox{ --\em #1}\fi}}}
\newcommand{\textcite}[1]{\citeauthor{#1} \shortcite{#1}}

\newcommand{\hide}[1]{}

\newcommand{\mtx}[1]{\mathbf{#1}}

\newcommand{\trans}{^\top}

\makeatletter
\newcommand{\iffont}[2]{\ifthenelse{\equal{\f@family}{#1}}{#2}{}}
\makeatother

  \usepackage{mathptmx}

  \DeclareSymbolFont{greek}{OML}{cmm}{m}{n}
  \DeclareMathSymbol{\alpha}{\mathalpha}{greek}{"0B}
  \DeclareMathSymbol{\beta}{\mathalpha}{greek}{"0C}
  \DeclareMathSymbol{\gamma}{\mathalpha}{greek}{"0D}
  \DeclareMathSymbol{\delta}{\mathalpha}{greek}{"0E}
  \DeclareMathSymbol{\epsilon}{\mathalpha}{greek}{"0F}
  \DeclareMathSymbol{\zeta}{\mathalpha}{greek}{"10}
  \DeclareMathSymbol{\eta}{\mathalpha}{greek}{"11}
  \DeclareMathSymbol{\theta}{\mathalpha}{greek}{"12}
  \DeclareMathSymbol{\iota}{\mathalpha}{greek}{"13}
  \DeclareMathSymbol{\kappa}{\mathalpha}{greek}{"14}
  \DeclareMathSymbol{\lambda}{\mathalpha}{greek}{"15}
  \DeclareMathSymbol{\mu}{\mathalpha}{greek}{"16}
  \DeclareMathSymbol{\nu}{\mathalpha}{greek}{"17}
  \DeclareMathSymbol{\xi}{\mathalpha}{greek}{"18}
  \DeclareMathSymbol{\pi}{\mathalpha}{greek}{"19}
  \DeclareMathSymbol{\rho}{\mathalpha}{greek}{"1A}
  \DeclareMathSymbol{\sigma}{\mathalpha}{greek}{"1B}
  \DeclareMathSymbol{\tau}{\mathalpha}{greek}{"1C}
  \DeclareMathSymbol{\upsilon}{\mathalpha}{greek}{"1D}
  \DeclareMathSymbol{\phi}{\mathalpha}{greek}{"1E}
  \DeclareMathSymbol{\chi}{\mathalpha}{greek}{"1F}
  \DeclareMathSymbol{\psi}{\mathalpha}{greek}{"20}
  \DeclareMathSymbol{\omega}{\mathalpha}{greek}{"21}
  \DeclareMathSymbol{\varepsilon}{\mathalpha}{greek}{"22}
  \DeclareMathSymbol{\vartheta}{\mathalpha}{greek}{"23}
  \DeclareMathSymbol{\varpi}{\mathalpha}{greek}{"24}
  \DeclareMathSymbol{\varrho}{\mathalpha}{greek}{"25}
  \DeclareMathSymbol{\varsigma}{\mathalpha}{greek}{"26}
  \DeclareMathSymbol{\varphi}{\mathalpha}{greek}{"27}
  \DeclareSymbolFont{otone}{OT1}{cmr}{m}{n}
  \DeclareMathSymbol{\Gamma}{\mathalpha}{otone}{0}
  \DeclareMathSymbol{\Delta}{\mathalpha}{otone}{1}
  \DeclareMathSymbol{\Theta}{\mathalpha}{otone}{2}
  \DeclareMathSymbol{\Lambda}{\mathalpha}{otone}{3}
  \DeclareMathSymbol{\Xi}{\mathalpha}{otone}{4}
  \DeclareMathSymbol{\Pi}{\mathalpha}{otone}{5}
  \DeclareMathSymbol{\Sigma}{\mathalpha}{otone}{6}
  \DeclareMathSymbol{\Upsilon}{\mathalpha}{otone}{7}
  \DeclareMathSymbol{\Phi}{\mathalpha}{otone}{8}
  \DeclareMathSymbol{\Psi}{\mathalpha}{otone}{9}
  \DeclareMathSymbol{\Omega}{\mathalpha}{otone}{10}
  \DeclareSymbolFont{syms}{OML}{cmm}{m}{it}
  \DeclareMathSymbol{\partial}{\mathord}{syms}{"40}
  \DeclareMathAlphabet{\mathbold}{OML}{cmm}{b}{it}
  \DeclareSymbolFont{largesymbols}{OMX}{cmex}{m}{n}

\newcommand{\insertMain}{
\begin{figure}[htbp]
\centering
{%
    \subfigure[]{%
    \includegraphics[width=0.49\linewidth]{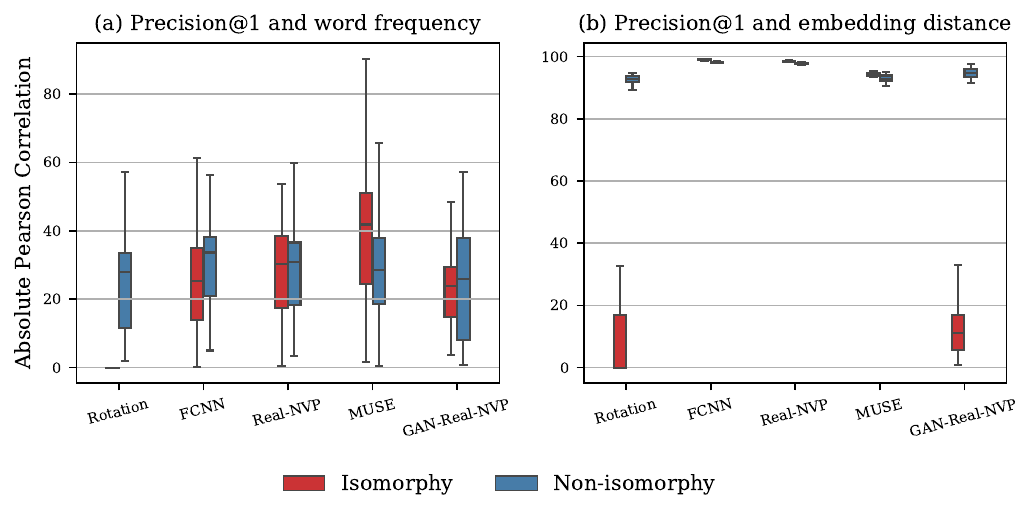}
    } 
    \subfigure[]{%
    \includegraphics[width=0.49\textwidth]{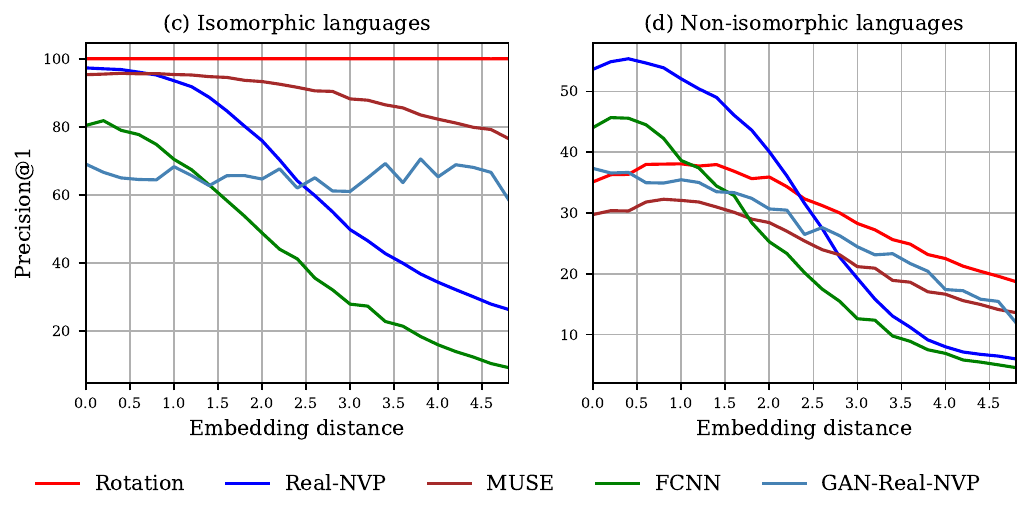}
    }
}
{\caption{Absolute Pearson correlation between task performance and (a) word frequency (occurrence) and (b) similarity between train and test domains (the distance between embeddings on train and test sets). We set frequency bins $k \in \{5,10,\dots,100\}$,
and similarity bins $\epsilon \in \{0, 0.2,\dots,5\}$.
We set $t$ to 1 in all isomorphic settings, and $t$ to 5 in non-isomorphic settings. (c)+(d) compares the generalization of approaches. Results are averaged across 10 runs.
}
\label{fig:main-simulation}
}
   
\end{figure} 
}

\newcommand{\insertValidationCriterion}{
\begin{figure}[htbp]
\centering
{%
    \subfigure{%
    \centering 
    \includegraphics[width=0.35\textwidth]{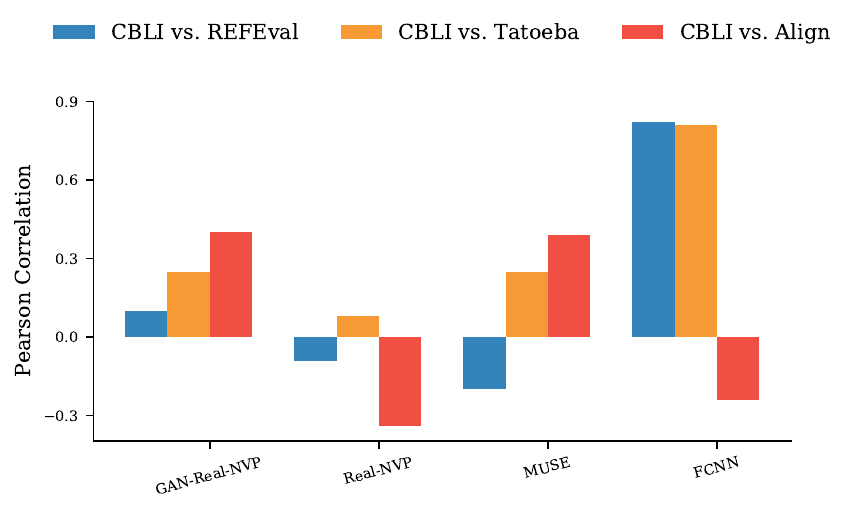}
    }
    \subfigure{%
    \centering
    \includegraphics[width=0.63\textwidth]{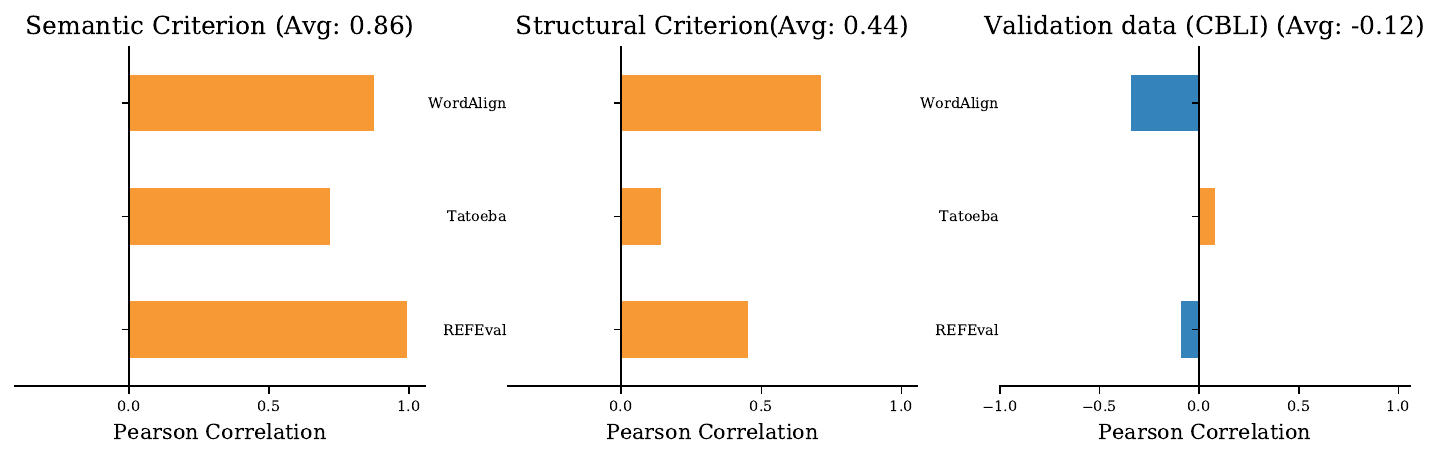}
    }
}
\caption{Pearson correlation between task performances (I) and between validation criteria and task performance given by Real-NVP (II). 
Results are averaged across languages and encoders. 
For each task, we collect model performances and criteria scores over 20 epochs.
\label{fig:validation-criterion}}  
   
\end{figure} 
}

\newcommand{\insertExample}{
\begin{figure*}[!htp]
\centerline{\includegraphics[width=0.6\linewidth]{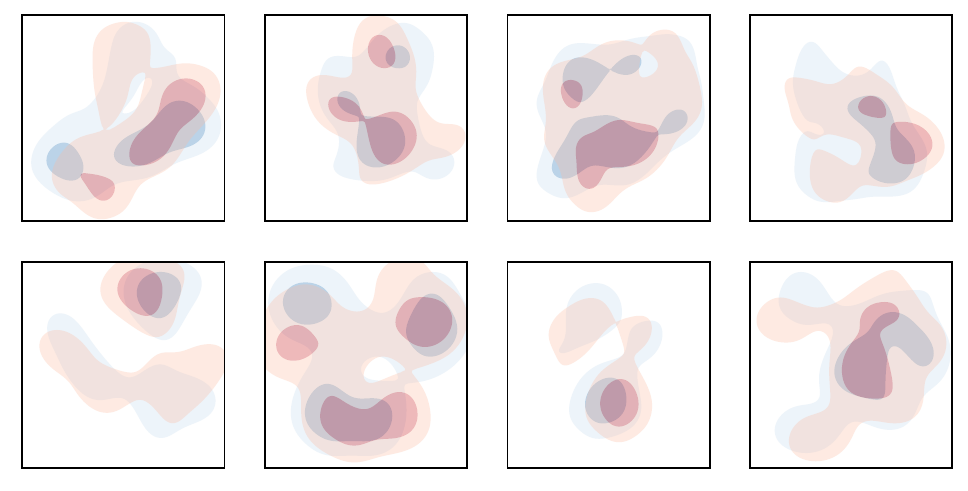}}  
 \caption{Eight figures are constructed in simulation. Each depicts two languages pertaining to two subspaces, colored in blue and red. Each subspace consists of up to 3 densities with each representing a word. Each density contains a number of data points sampled from a two-dimensional Gaussian distribution, as a reflection of word occurrence.
 }
 \label{fig:example}  
\end{figure*} 
}

\newcommand{\insertDensity}{
\begin{figure}[htbp]
\centering
{%
    \subfigure[]{%
    \centering 
    \includegraphics[width=0.43\linewidth]{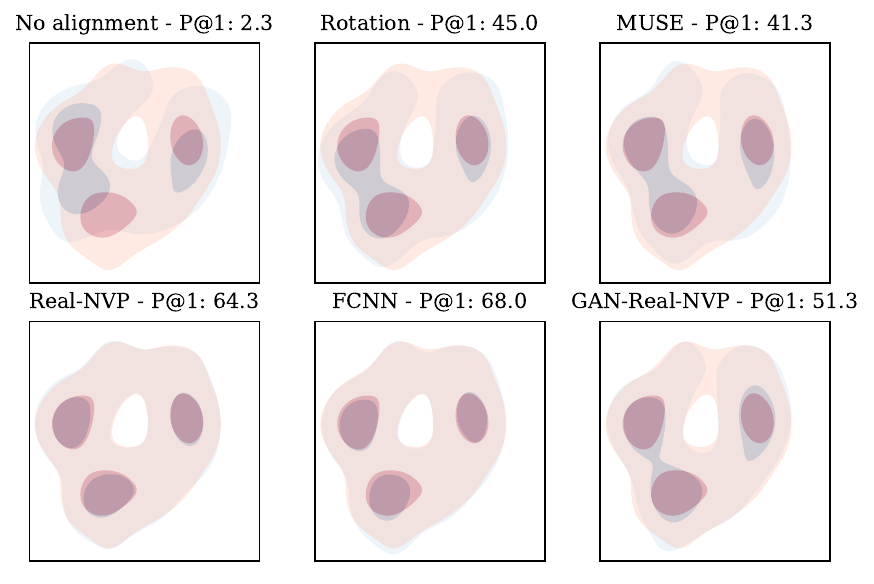}
    }
    \subfigure[]{%
    \centering
    \includegraphics[width=0.54\linewidth]{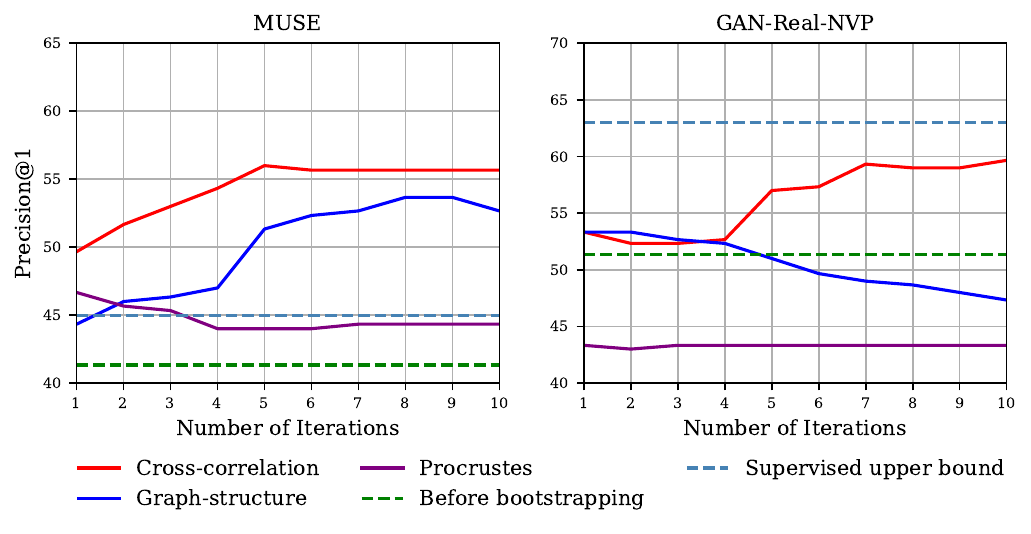}
    }
}
{\caption{(I) shows how well two languages are aligned according to a visual introspection (subspace overlaps) and Precision@1;
 (II) compares unsupervised approaches (MUSE and GAN-Real-NVP) with the supervised counterparts (Rotation and Real-NVP) in non-isomorphic settings ($t=5$). We set the occurrence per word $k$ to 100. 
 }\label{fig:density-fine-tuning}  }
 
\end{figure} 
}

\newcommand{\insertTableTaskPerformance}{
\begin{table}[htbp]
    \footnotesize
    \setlength{\tabcolsep}{3.2pt}
      \centering
      {
        \begin{tabular}{l rrrr| rrrr}
            \toprule
            & \multicolumn{4}{c}{m-BERT} & \multicolumn{4}{c}{XLM-R} \\
            Alignments & RFEval & Tatoeba & CBLI & Align & RFEval & Tatoeba & CBLI & Align\\
            \midrule
            Original & 27.23&49.35&50.9&61.54&26.42&63.40&48.58&59.77\\
            \midrule
            \multicolumn{8}{l}{\textit{Supervised mapping functions}}\\
            Rotation-20k & 38.73&55.28&58.45&\textbf{62.83}&34.67&68.60&53.67&60.85\\
            FCNN-20k (semantic criterion) & \textbf{42.72}&61.18&55.30&61.50&36.67&\textbf{80.20}&50.88&59.87\\
            FCNN-20k (5 epochs) & 38.40&58.97&54.02&61.02&33.50&77.98&50.48&59.50\\
            Real-NVP-20k (semantic criterion) & \textbf{42.32}&\textbf{62.87}&57.62&\textbf{62.59}&\textbf{44.17}&\textbf{80.08}&\textbf{61.63}&\textbf{62.84}\\
             Real-NVP-20k (5 epochs) & 40.12&60.70&58.52&61.80&42.24&78.75&61.76&61.20\\
            GBDD-20k & 28.77&52.28&51.42&61.71&27.13&68.85&48.42&59.81\\
            Joint-Align-100k (3 epochs) & 41.23&59.13&\textbf{64.67}&\textbf{62.30}&-&-&-&- \\
            InfoXLM-42GB (150K training steps) &-&-&-&- &37.60&76.10&60.80&\textbf{62.94} \\
            \midrule
            \multicolumn{8}{l}{\textit{Unsupervised mapping functions}}\\
            Normalization & 30.08&61.28&54.88&\textbf{62.54}&39.52&79.75&59.03&62.55\\
            VecMap-20k ($\sim$500 epochs) & 30.77&55.00&\textbf{64.42}&\textbf{62.50}&-&-&-&-\\
            MUSE-20k (5 epochs) & 29.20&50.20&52.30&61.56&25.21&63.42&50.20&60.20\\
            MUSE-20k (semantic criterion) & 31.23&51.42&52.48&61.64&27.55&65.72&50.00&60.01\\
             + Cross-Correlation & 35.25&52.90&52.87&\textbf{62.63}&32.05&69.23&49.80&60.49 \\
             + Graph Structure& 33.22&51.65&53.10&62.17&29.48&68.33&50.18&60.46 \\
             + Procrustes & 36.85&54.13&55.22&\textbf{62.71}&33.37&68.82&50.37&60.59\\
            \hdashline\noalign{\vskip 0.5ex}
            GAN-Real-NVP-20k (5 epochs) & 32.24&59.10&56.79&61.80&39.61&77.77&60.83&60.90\\
            GAN-Real-NVP-20k (semantic criterion) & 33.90&61.20&57.03&\textbf{62.33}&41.72&79.67&\textbf{61.00}&\textbf{62.81}\\
            + Cross-Correlation & 35.33&\textbf{62.32}&58.00&\textbf{62.70}&\textbf{42.60}&\textbf{80.50}&\textbf{61.23}&\textbf{63.15} \\
            + Graph Structure & 34.32&61.82&56.65&\textbf{62.52}&41.55&80.02&60.83&\textbf{62.99} \\
            + Procrustes & \textbf{36.93}&53.95&56.05&\textbf{62.79}&33.78&67.95&51.60&60.51 \\
        	\bottomrule  
        \end{tabular}
   }
    \caption{Results are averaged across language pairs. We bold numbers that significantly outperform others according to paired t-test.
    Joint-Align uses 100k parallel data per language pair; others only use 20k data. InfoXLM uses 42GB parallel data in total.
    Rotation, GBDD and Normalization with closed-form solutions do not require validation criteria for model selection.
    \label{tab:main-task-performances} 
    }
\end{table}
}

\newcommand{\insertTaskCorrelation}{
\begin{table*}[!htbp]
\centering
\footnotesize
\begin{tabular}{ l | l | c c c }
\toprule
Settings &Alignments   & $[-1, 0]$ & $(0, 0.4]$ & $(0.4, 1]$ \\
\midrule
\multirow{2}{*}{Supervised}&
FCNN-20k   & $50\%$ & $0\%$ & $50\%$ \\
& Real-NVP-20k   & $33\%$ & $17\%$ & $50\%$ \\
\midrule
\multirow{2}{*}{Unsupervised} &
MUSE-20k   & $17\%$ & $66\%$ & $17\%$ \\
&GAN-Real-NVP-20k  & $17\%$ & $50\%$ & $33\%$ \\
\bottomrule
\end{tabular}
\caption{
Correlation statistics: the last three columns split the Pearson's $\rho$ range into three intervals. Each entry denotes the percent of task pairs in which the correlation between model performances is in one of the intervals. For instance, the performances across tasks exhibit negative correlations in 17\%-50\% task pairs. For each task, we collect the model performances over 20 epochs. Results are averaged across language pairs and encoders.
\label{tab:task-correlation}}
\end{table*}
}
\usepackage[utf8]{inputenc}
\usepackage{microtype}
\newcommand{\wei}[1]{\textcolor{black}{#1}}

\makeatletter
\let\Ginclude@graphics\@org@Ginclude@graphics 
\makeatother

\jmlrvolume{}
\jmlryear{2022}
\jmlrworkshop{ACML 2022}

\title[Constrained Density Matching and Modeling for Cross-lingual Alignment]{Constrained Density Matching and Modeling for Cross-lingual Alignment of Contextualized Representations}

  \author{\Name{Wei Zhao} \Email{wei.zhao@h-its.org}\\
  \addr Heidelberg Institute for Theoretical Studies \\ Fachbereich Informatik, Technische Universit\"at Darmstadt 
  \AND
  \Name{Steffen Eger} \Email{steffen.eger@uni-bielefeld.de}\\
  \addr Technische Fakult\"at, Universit\"at Bielefeld \\ Fachbereich Informatik, Technische Universit\"at Darmstadt
 }
       
\editors{Emtiyaz Khan and Mehmet Gonen}

\begin{document}

\maketitle

\begin{abstract}
Multilingual representations pre-trained with monolingual data exhibit considerably unequal 
task performances across languages. 
Previous studies address this challenge with resource-intensive contextualized alignment, which assumes the availability of large parallel data, thereby leaving under-represented language communities behind.
In this work, we attribute the data hungriness of previous alignment techniques to two limitations: (i) the inability to sufficiently leverage data and (ii) these techniques are not trained properly. 
To address these issues, we introduce supervised and unsupervised density-based approaches named Real-NVP and GAN-Real-NVP, driven by Normalizing Flow, to perform alignment, both dissecting the alignment of multilingual subspaces into density matching and density modeling. We complement these approaches with our validation criteria in order to guide the training process. 
Our experiments encompass 16 alignments, including our approaches, evaluated across 6 language pairs, synthetic data and 5 NLP tasks.
We demonstrate the effectiveness of our approaches in the scenarios of limited and no parallel data. First, our supervised approach trained on 20k parallel data (sentences) mostly surpasses Joint-Align and InfoXLM trained on over 100k parallel sentences. Second, parallel data can be removed without sacrificing performance when integrating our unsupervised approach in our bootstrapping procedure, which is 
theoretically motivated to enforce equality of multilingual subspaces. 
Moreover, we demonstrate the advantages of validation criteria over validation data for guiding supervised training\footnote{Our code and models are available at \url{https://github.com/AIPHES/DensityAlign}}.

\end{abstract}

\section{Introduction}
\label{sec:intro}

Multilingual text encoders such as m-BERT~\citep{devlin-etal-2019-bert} and XLM-R~\citep{conneau-etal-2020-unsupervised}
have been profiled as the de facto solutions to modeling languages at scale. However, research showed that such encoders pre-trained with monolingual data have failed to align multilingual subspaces, and exhibit strong language bias, i.e., the quality of these encoders largely differs across languages, particularly for dissimilar and low-resource languages~\citep{pires:2019, zhao:2020}.

For that reason, supervised alignment techniques 
emerged, aiming to rectify multilingual representations post-hoc with cross-lingual supervision~\citep{Cao:2020, zhao:2020-refeval, infoxlm}, 
but previous studies are limited in scope to high-resource languages requiring large-scale parallel data. In contrast, unsupervised alignment removing the dependence on parallel data
allows for unlimited 
use
in all languages \citep{artetxe-etal-2017-learning, vecmap:2018}. However, previous studies of unsupervised alignment focusing on static embeddings have not touched on contextualized representations.

In this work, we address cross-lingual alignment for contextualized representations, termed \emph{contextualized alignment}, particularly in the scenarios of limited and no parallel data. In \textbf{supervised} settings, we identify two limitations responsible for the ineffectiveness of previous resource-intensive contextualized alignments: (i) the inability to sufficiently leverage data, i.e., that these techniques do not target the modeling of data density, and (ii) that they are not properly trained due to a lack of validation criteria---recent techniques, such as \citet{wu-dredze-2020-explicit} and \citet{Cao:2020}, have been trained for several epochs without  access to any criteria for model selection, coming at the risk of being mistrained. 
To this end, we start by introducing a density-based, contextualized alignment, which dissects the alignment of multilingual subspaces into two sub-problems with one solution: density modeling and density matching, addressed by Normalizing Flows~\citep{Dinh:2017}.
Second, in order to guide the training process, we present two validation criteria for model selection during training, and demonstrate the superiority of these criteria over validation data.

In \textbf{unsupervised settings}, aiming for unsupervised, contextualized alignment, we carry out density modeling and density matching in the form of adversarial learning~\citep{goodfellow:2014}, and complement this learning process with the validation criteria mentioned previously to guide unsupervised training. Further, we identify a statistical issue of density matching in the unsupervised case: density matching only leads to a weak notion of equality of multilingual subspaces, \emph{viz.}, equality in distribution. Accordingly, we present a bootstrapping procedure enhancing unsupervised alignment by promoting equality of multilingual subspaces. 
We evaluate 
our approaches across 6 language pairs, synthetic data and 5 NLP tasks. Our major findings are summarized as follows:

\begin{itemize}
    \item 
    With 20k parallel data we provided, our supervised alignment mostly surpasses Joint-Align \citep{Cao:2020} and InfoXLM \citep{infoxlm} trained on much larger parallel data. This confirms the effectiveness of the conflation of density matching and density modeling as our alignment does. Second, our unsupervised alignment integrated in bootstrapping procedure rivals supervised counterparts, showing that parallel data can be removed without sacrificing performance. But we admit that these alignments, be it supervised or not, are poor in generalization (see \S\ref{sec:simulation}), calling for an improvement in future work.
    
    \item 
    Not only are validation criteria crucial for guiding unsupervised training, but also for supervised training. 
    Given the performance on validation data and external tasks often correlates weakly, validation data is inappropriate for guiding supervised training. Above all, guiding contextualized alignment with validation criteria is challenging, as the model performances across tasks exhibit negative correlations in about $30\%$ setups in our experiments, i.e., the better the alignment performs in one task, the worse it performs in the other. Thus, we base the evaluation of validation criteria on the model performances in all tasks. 
    We find that validation criteria correlate much better than validation data (treated as criterion) with model performance on average across tasks for guiding supervised training.
\end{itemize}

\section{Related Work}
\label{sec:related}

Recent advances in multilingual representations, such as m-BERT and XLM-R, boost the performance of cross-lingual NLP systems. However, such systems 
exhibit weak(er) performance for dissimilar languages~\citep{pires:2019}
and 
low-resource languages~\citep{zhao:2020}. 
Accordingly, contextualized alignment emerged.
\citet{aldarmaki:2019} show that language-dependent rotation can linearly rectify m-BERT representations. \citet{Cao:2020} find that jointly aligning multiple languages performs better. \citet{zhao:2020-refeval} show that removing language bias in multilingual representations mitigates the vector space misalignment between languages. More recently,  \citet{xia2021metaxl} show that gradient-based alignment is effective for the languages not covered during pre-training in XLM-R. \citet{ot:2021} use optimal transport to finetune multilingual representations, while \citet{infoxlm}
finetune them with translation language modeling as the learning objective. 
However, these studies focused on supervised, resource-intensive alignment techniques and required from 250k to ca. 2M parallel sentences (or a large-scale analogy corpus) in each language pair for substantial improvement. As early attempts to remove the use for parallel data, \citet{libovicky-etal-2020-language} and \citet{zhao:2020} find applying vector space normalization is helpful to yield language-neural representations. However, there lacks a thorough study on unsupervised, contextualized alignment for multilingual representations. 

As for unsupervised alignment, previous studies have predominantly focused on static embeddings, which mostly rely on iterative procedures in two steps, aiming to derive bilingual lexicons as cross-lingual supervision:
(i) inducing seed dictionaries with different approaches, such as adversarial learning~\citep{lample:2018}, similarity based heuristics~\citep{vecmap:2018} and identical strings~\citep{artetxe-etal-2017-learning}, and (ii) applying Procrustes to augment induced lexicons~\citep{lample:2018} 
in an iterative fashion. 

In this work, we present a principled, iterative procedure to enhance our unsupervised alignment on contextualized representations, which employs density-based approaches to induce bilingual lexicons, and then applies our bootstrapping procedure, theoretically grounded in statistics for equality of multilingual subspaces, to iteratively augment lexicons. Lastly, we complement the iterative procedure with validation criteria to guide unsupervised training. We contrast our approaches with other domain adaptation techniques in Section \ref{sec:broader-impact}.
\section{Contextualized Alignment}

Let two random variables $X$ and $Y$ with densities $P_X$ and $P_Y$ describe two populations of contextual word embeddings pertaining to two languages $\ell_1$ and $\ell_2$, with $\Omega_{\ell_1}$ and $\Omega_{\ell_2}$ as two lexicons. Each occurrence of a word is associated to a separate entry in the lexicons.
$X$ maps all entries in $\Omega_{\ell_1}$ to real-valued $m$-dimensional embedding vectors, denoted by $X:\Omega_{\ell_1} \rightarrow \mathbb{R}^m$, and similarly for $Y$. 
A bilingual lexicon $\Omega$ describes a set of translations between $\Omega_{\ell_1}$ and $\Omega_{\ell_2}$.

\paragraph{Empirical inference.} 
Assume a function $f: \mathbb{R}^m \rightarrow \mathbb{R}^m$ perfectly maps $m$-dimensional embedding vectors from $X$ to $Y$. As standard in machine learning, a mapping function $f_\theta$ can be empirically inferred from data, with $\theta$ as model parameters. To this end, we assume data $\mtx{M}_{\ell_1} \in \mathbb{R}^{n\times m}$ and $\mtx{M}_{\ell_2} \in \mathbb{R}^{n\times m}$ are given, corresponding to two sets of contextual word embeddings with a common size of $n$ for simplicity. 
Let a permutation matrix $\mtx{P} \in \{0,1\}^{n\times n}$ ($\mtx{P}\mtx{1}_n=\mtx{1}_n$ and $\mtx{P}\trans{\mtx{1}}_n=\mtx{1}_n$) be a realization of $\Omega$, serving as cross-lingual supervision when available. 
A random variable $\tilde{Y}$with density $P_{\tilde{Y}}$ is a prediction of $Y$ given $X$, i.e., $\tilde{Y}= f_{X\rightarrow Y}(X)$ where the subscript denotes mapping direction.

\subsection{Supervised Alignment}
When parallel data is available, a permutation matrix $\mtx{P}$ can be effortless induced from parallel data with word alignment tools \citep{dyer:2013, simalign:2020}. We introduce a density-based mapping function focusing on two components: density matching and density modeling. To do so, we start by depicting the alignment of multilingual subspaces in the form of density matching between $P_{\Tilde{Y}}$ and $P_Y$:
\begin{equation}
\begin{split}
\label{eq:kl}
\mathrm{KL}(P_{\Tilde{Y}}, P_Y) &= \mathrm{CE}(P_{\Tilde{Y}}, P_Y) - \mathbb{E}_{y\sim P_Y}[\log P_Y(y)] \\ 
&= \|f_{X\rightarrow Y}(\mtx{M}_{\ell_1}) - \mtx{P}\mtx{M}_{\ell_2} \|^2 
- \mathbb{E}_{y\sim P_Y}[\log P_X(f^{-1}_{X\rightarrow Y}(y))|\mathrm{det}(\triangledown_\theta f^{-1}_{X\rightarrow Y}(y))|] 
\end{split}
\end{equation}
where $f_{X\rightarrow Y}$ is the trainable mapping function from $X$ to $Y$.
Given $P_Y$ intractable to compute, previous supervised alignments always minimize the cross-entropy term alone by solving the least squares problem.
Note that the density $P_Y$ can be rewritten to $P_X(f^{-1}_{X\rightarrow \Tilde{Y}}(y))|\mathrm{det}(\triangledown_\theta f^{-1}_{X\rightarrow \Tilde{Y}}(y))|$ by using the change-of-variable rule (assuming $f$ is an invertible function). However, the density $P_Y$ is still intractable given the unknown density $P_X$. 
We overcome this by introducing a generative model named Real-NVP~\citep{Dinh:2017} as use case of Normalizing Flows~\citep{nf:2015}. \wei{Real-NVP is a popular example of invertible neural networks, which can be thought of as a bijective function between two domains of data points (e.g., random noise and real data). Here, we use Real-NVP to address density estimation, i.e., inferring the unknown distribution of word embeddings $X$ and $Y$ from a normal distribution of random noise via the change-of-variable rule.}

To do so, we introduce a latent variable $Z \sim \mathcal{N}(0,\mtx{I})$ with the normal density $P_Z$ \wei{to describe random noise}.
We then use Real-NVP to infer $P_Y$
from $P_Z$,
denoted by $P_Y(y)=P_Z(f_{Z\rightarrow Y}^{-1}(y))|\mathrm{det}(\triangledown_\theta f_{Z\rightarrow Y}^{-1}(y))|$ with $f_{Z\rightarrow Y}$ as a trainable mapping function from $Z$ to $Y$.
Lastly, we rewrite the entropy term in Eq.~\ref{eq:kl} to:
\begin{equation}
\begin{split}
\label{eq:nf-density}
    &\mathbb{E}_{y\sim P_Y}[\log P_Y(y)] =\mathbb{E}_{y\sim P_Y}[\log \mathcal{N}(f_{Z\rightarrow Y}^{-1}(y), 0,\mtx{I})] 
    + \mathbb{E}_{y\sim P_Y}[\log |\mathrm{det}(\triangledown_\theta f_{Z\rightarrow Y}^{-1}(y))|]
\end{split}
\end{equation}

To consider density estimation (modeling) on both $P_X$ and $P_Y$, we perform a dual form of density matching based on JS divergence. We omit the definition for simplicity. In \S\ref{sec:results}, we refer to the above described approach as Real-NVP. 

\subsection{Unsupervised Alignment}
\label{sec:unsupervised}
When $\mtx{P}$ is not given due to a lack of parallel data, we apply adversarial learning to align the two densities $P_{\Tilde{Y}}$ and $P_Y$. As standard in adversarial training, we involve a min-max game between two components to perform density matching: (a) a discriminator distinguishing source and target word embeddings after mapping them and (b) a mapping function aligning source and target word embeddings in order to fool the discriminator. We use a popular adversarial approach, the Wasserstein GAN \citep{wgan:2017}, which aligns the densities $P_{\tilde{Y}}$ and $P_Y$ by minimizing the Earth Mover distance (EMD) between these densities.  To better leverage data, we include density estimation (modeling) based on Real-NVP in the procedure of adversarial training, which maximizes the data likelihood of $X$ and $Y$. 
Taken together, we denote our density-based learning objective by:
\begin{equation}
\begin{split}
\label{eq:wgan}
 &\mathrm{EMD}(P_{\tilde{Y}}, P_Y) =\min_{f_{X\rightarrow Y}} \max_{h_\phi} \mathbb{E}_{y \sim P_Y}[h_{\phi}(y)] 
 - \mathbb{E}_{\tilde{y} \sim P_{\tilde{Y}}}[h_{\phi}(\tilde{y}))]
  + \mathbb{E}_{y\sim P_Y}[\log P_Y(y)]
\end{split}
\end{equation}  
where $h_\phi$ is a 1-Lipschitz constrained discriminator, $\tilde{y}=f_{X\rightarrow Y}(x)$ mapping $X$ to $Y$. Note that $f_{X\rightarrow Y}$ is the composition of $f_{X\rightarrow Z}$ and $f_{Z\rightarrow Y}$, and
the last entropy term aims to maximize the data log-likelihood of $Y$. As in the supervised case, we use a dual form of Eq.~\ref{eq:wgan}. In \S\ref{sec:results}, we refer to the above described approach as GAN-Real-NVP. 

\paragraph{Bootstrapping procedure.} 
After adversarial learning $Y$ and $\tilde{Y}$ are ideally equal in distribution, denoted by $Y\stackrel{\mathrm{dist}}{=} \Tilde{Y}$. However, this is not sufficient.
For instance, let $Y \sim \mathrm{Uniform}(-1, 1)$ and $\Tilde{Y} = -Y$. Clearly, 
$Y$ and $Y$ are equal in distribution, but they are identical only at the origin. Here, we derive the two following conditions that promote the equality of $Y$ and $\Tilde{Y}$ and enhance unsupervised alignment.

\begin{proposition} 
\label{prop:refinement}
Given $Y\stackrel{\mathrm{dist}}{=} \Tilde{Y}$, $\Tilde{Y}$ and $Y$ are equal if one of the following conditions is met:
\begin{itemize}
\item[(i)]
$Y = \mtx{U} \Tilde{Y}$, where $\mtx{U}$ is invertible and $\mtx{U}_{ij}\geq0$ $\forall i,j$. 

\item[(ii)] 
$\mathrm{cor}(\Tilde{Y}_i, Y_i)=1$ for $\forall i$, where $\Tilde{Y}_i$ and $Y_i$ represent the $i$-th component in $\Tilde{Y}$ and $Y$.

\end{itemize}
\end{proposition}

\begin{proof}

(i) 
$P(Y\leq y) = P(\Tilde{Y}\leq y)$ for all $y$, due to $Y\stackrel{\mathrm{dist}}{=} \Tilde{Y}$. 
If $Y = \mtx{U}\Tilde{Y}$, then $P(\Tilde{Y}\le y)=P(Y\le y)=P(\mathbf{U}\Tilde{Y}\le y)$. 
If $\mathbf{U}\ge 0$, then $P(\mathbf{U}\Tilde{Y}\le y)=P(\Tilde{Y}\le \mathbf{U}^{-1}y)$. 
Thus, $P(\Tilde{Y}\le \mathbf{U}^{-1}y)=P(\Tilde{Y}\le y)$ for all $y$. 
This implies that $\mtx{U} = \mtx{I}$. Thus, $\Tilde{Y} = Y$.

(ii) If $\mathrm{cor}(\Tilde{Y}_i, Y_i)=1$ for $\forall i$, then $\mathrm{Var}[(\frac{\Tilde{Y}_i}{\sigma_{\Tilde{y}_i}} - \frac{Y_i}{\sigma_{y_i}})] = 0$, thus $\mathbb{E}[(\frac{\Tilde{Y}_i}{\sigma_{\Tilde{y}_i}} - \frac{Y_i}{\sigma_{y_i}})^2] - {\mathbb{E}[(\frac{\Tilde{Y}_i}{\sigma_{\Tilde{y}_i}} - \frac{Y_i}{\sigma_{y_i}})]}^2 = 0$. However, the second term equals to $0$ by using $\mathbb{E}[(\frac{\Tilde{Y}_i}{\sigma_{\Tilde{y}_i}} - \frac{Y_i}{\sigma_{y_i}})] = \frac{\mathbb{E}[\Tilde{Y}_i]}{\sigma_{\Tilde{y}_i}} - \frac{\mathbb{E}[Y_i]}{\sigma_{y_i}} = 0$ due to $\Tilde{Y}_i \stackrel{\mathrm{dist}}{=}Y_i$.
Thus, $\mathbb{E}[(\frac{\Tilde{Y}_i}{\sigma_{\Tilde{y}_i}} - \frac{Y_i}{\sigma_{y_i}})^2]=0$, 
and this implies $\frac{\Tilde{Y}_i}{\sigma_{\Tilde{y}_i}} = \frac{Y_i}{\sigma_{y_i}}$ since the non-negative $(\frac{\Tilde{Y}_i}{\sigma_{\Tilde{y}_i}} - \frac{Y_i}{\sigma_{y_i}})^2$ must be zero if its expectation is $0$. Note that $\sigma_{\Tilde{y}_i}=\sigma_{y_i}$ due to $\Tilde{Y}_i \stackrel{\mathrm{dist}}{=}Y_i$. This implies that $\Tilde{Y}_i = Y_i$ for $\forall i$. 
\end{proof}

To design computational approaches meeting the above conditions,
we introduce additional notation and the following lemma.
Let $\mtx{M}_{X}$, $\mtx{M}_Y$ be embeddings from $X$ and $Y$, and  $\mtx{M}_{\Tilde{Y}}$=$f_\theta(\mtx{M}_{X})$.

\begin{lemma}
\label{eq:lemma}
If $\mtx{M}_{\Tilde{Y}}\mtx{M}_{\Tilde{Y}}^{\intercal} = \mtx{M}_Y\mtx{M}_Y^{\intercal}$ and $\mtx{M}_{\Tilde{Y}}$ is invertible, then $Y=\mtx{U}\tilde{Y}$.
\end{lemma}
\begin{proof}
If $\mtx{M}_{\Tilde{Y}}\mtx{M}_{\Tilde{Y}}^{\intercal} = \mtx{M}_Y\mtx{M}_Y^{\intercal}$ and $\mtx{M}_{\Tilde{Y}}$ is invertible, then $\mtx{M}_Y = \mtx{M}_{\Tilde{Y}} \mtx{M}_{\Tilde{Y}}^{-1}\mtx{M}_{Y}$. Let $\mtx{U}=\mtx{M}_{\Tilde{Y}}^{-1}\mtx{M}_{Y}$. Then, $\mtx{M}_Y = \mtx{M}_{\Tilde{Y}} \mtx{U}$. If this holds for all $\mtx{M}_{\Tilde{Y}}$ and $\mtx{M}_Y$, then $Y=\mtx{U}\tilde{Y}$.
\end{proof}

In the following, we describe our computational approaches, and then include them as \textbf{constraints} in the adversarial training in order to promote the equality of $Y$ and $\tilde{Y}$. Lastly, we discuss the connection of these constraints with canonical correlation and language isomorphism.

\paragraph{Graph structure.} We depict $\mtx{M}_{\Tilde{Y}}$ and $\mtx{M}_Y$ as $m$-dimensional vertices in two graphs, with $\mtx{M}_{\Tilde{Y}}\mtx{M}_{\Tilde{Y}}^{\intercal}$ and $\mtx{M}_Y\mtx{M}_Y^{\intercal}$ as the weighted adjacency matrices on these graphs. As Lemma~\ref{eq:lemma} states, minimizing the difference between these adjacency matrices allows to meet Prop.\ref{prop:refinement}(i). Thus, the objective becomes:

\begin{equation}
\label{eq:graph}
    \mathrm{EMD}(P_{\tilde{Y}}, P_Y) + 
    \|\mtx{M}_{\Tilde{Y}}\mtx{M}_{\Tilde{Y}}^{\intercal} -  \mtx{M}_Y\mtx{M}_Y^{\intercal}\|^2
\end{equation}
However, we admit that Prop.\ref{prop:refinement}(i) cannot  be strictly met, as guaranteeing $\mtx{U}\geq 0$, i.e., $\mtx{M}_{\Tilde{Y}}^{-1}\mtx{M}_{Y}\geq 0$ is not trivial. This might explain why \emph{graph structure} is worse than \emph{cross-correlation} in our experiments.

\paragraph{Cross-correlation.} We maximize Pearson cross-correlation between de-meaned $\mtx{M}_{\Tilde{Y}}$ and $\mtx{M}_Y$ in order to realize Prop.\ref{prop:refinement}(ii). The objective becomes: 
\begin{equation}
\label{eq:correlation}
    \mathrm{EMD}(P_{\tilde{Y}}, P_Y) + 
    \|\frac{\mathrm{diag}(\mtx{M}_{\Tilde{Y}}^{\intercal}\mtx{M}_Y)}{\mathrm{diag}(\mtx{M}_{\Tilde{Y}}^{\intercal}\mtx{M}_{\Tilde{Y}})\mathrm{diag}(\mtx{M}_Y^{\intercal}\mtx{M}_Y)} - \vv{1}\|^2
\end{equation}

Concerning the construction of $\mtx{M}_{Y}$ and $\mtx{M}_{\tilde{Y}}$, 
we use CSLS~\citep{lample:2018} to induce them from monolingual data, and then update them in an iterative fashion with 
Algorithm~\ref{algo:procedure}.

\begin{algorithm2e}
  \caption{Bootstrapping Procedure}\label{algo:procedure}
\KwIn{$\mtx{M}_{X}, \mtx{M}_{Y} \leftarrow $ population word embeddings of $X$ and $Y$}
\KwIn{$n \leftarrow$ number of bootstrapping iterations \Comment{simulation: $n=10$; real data: $n=3$}}
\KwIn{$f_{X\rightarrow Y} \leftarrow $ an identity matrix as initial mapping function}
\For{$i\leftarrow 1$ \KwTo $n$}{
$\mtx{M}_{\tilde{Y}} \leftarrow f_{X\rightarrow Y}(\mtx{M}_{X})$

$\mtx{P} \leftarrow \mathrm{CSLS}(\mtx{M}_{Y}, \mtx{M}_{\tilde{Y}})$\Comment{induce permutation matrix}

$f_{X\rightarrow Y} \leftarrow \mathrm{EMD}(P_{\tilde{Y}}, P_Y) + g(\mtx{M}_{Y}, \mtx{P}\mtx{M}_{\tilde{Y}})$ \Comment{$g$ is a function of our constraints (see Eq.~\ref{eq:graph}+\ref{eq:correlation})}
}
\KwOut{$f_{X\rightarrow Y}$}
\end{algorithm2e}

\paragraph{Connection with canonical correlation.}
Often, cross-correlation between random vectors are computed using Canonical Correlation Analysis (CCA). Research showed that CCA is useful to improve static embeddings, but it requires finding $k$ primary canonical variables \citep{faruqui:2014}. In contrast, our solution is much cheaper to compute cross-correlation without the need for canonical variables (see Eq.~\ref{eq:correlation}).

\paragraph{Connection with language isomorphism.}
In graph theory, two graphs are called isomorphic when the two corresponding adjacency matrices are permutation similar. According to Eq.~\ref{eq:graph}, our solution aims to minimize the difference between adjacency matrices, and as such lays the foundation of  graph isomorphism---which is termed language isomorphism in the multilingual community. Taken together, our solution allows for yielding isomorphic multilingual subspaces for non-isomorphic languages such as typologically dissimilar languages.

\section{Experiments}
\label{sec:results}
 
\subsection{Baselines and Our Approach}

\paragraph{Supervised alignments.} 
(a) Rotation~\citep{aldarmaki:2019, zhao:2020}: a linear orthogonal-constrained transformation;
(b) GBDD~\citep{zhao:2020-refeval}: subtracting a global language bias vector from multilingual representations; (c) FCNN: an architecture that contains three fully-connected layers followed by a tanh activation function each; (d) Joint-Align~\citep{Cao:2020}: jointly aligning many languages via fine-tuning; (e) InfoXLM~\citep{infoxlm}: finetuning multilingual representations with translation language modeling and contrastive learning; (f) our Real-NVP.

\paragraph{Unsupervised alignments.} 
(a) MUSE: the unsupervised variant of Rotation~\citep{lample:2018};
(b) VecMap: a heuristic unsupervised approach based on the assumption that word translations have similar distributions on word similarities~\citep{vecmap:2018};
(c) vector space normalization~\citep{zhao:2020}: removing language-specific means and variances of multilingual representations.
MUSE and VecMap are popular 
unsupervised alignments on static embeddings; (d) our GAN-Real-NVP.
For bootstrapping procedure, we use the notation: [Method]$+$[Constraint], where [Method] is MUSE or GAN-Real-NVP, and [Constraint] is Cross-Correlation or Graph-Structure or Procrustes---known to enhance unsupervised alignment on static embeddings. 

Except for Joint-Align, InfoXLM, and Normalization, the others are trained individually across language pairs.

\subsection{Validation Criterion}
We present two validation criteria, and compare them with no-criteria (i.e., training for several epochs) in both supervised and unsupervised settings. In particular, we induce the 30k most confident word translations from monolingual data with CSLS, and then compute the two following criteria on these word translations.

\begin{itemize}
    \item \textit{Semantic criterion} was proposed for guiding the training of unsupervised alignment on static embeddings. \citet{lample:2018} assemble the 10k most frequent source words and generate target translations of these words. Next, they average cosine similarities on these translation pairs treated as validation criterion.
    
    \item \textit{Structural Criterion}: we compute the difference between two ordered lists of singular values obtained from source and target word embeddings pertaining to the 30k most confident word translations. This criterion was initially proposed to measure language isomorphism \citep{dubossarsky:2020}.

\end{itemize}

\subsection{Simulation}
\label{sec:simulation}

Bilingual Lexicon Induction (BLI) is a popular internal task known to evaluate alignment on static embeddings, as it covers ca. 100 language pairs and focuses on the understanding of the alignment itself other than its impact on external tasks. In particular, BLI bases the induction of bilingual lexicons on static word embeddings, and compares the induced lexicons with gold lexicons.

However, contextual embeddings lack such evaluation tasks.
As \citet{artetxe-etal-2020-call} state, when not evaluated under similar conditions, the lessons learned from static embeddings cannot transfer to contextual ones. To this end, we perform simulation to construct synthetic data as the contextual extension of BLI (CBLI), which focuses on evaluating the alignment of multilingual subspaces of contextualized embeddings.

We split CBLI data to train, validation and test sets, and report Precision@K, as in BLI evaluation.
Our creation procedure is two-fold: First, we sample source  embeddings from a two-dimensional Gaussian (normal) mixture distribution, and then perform different 
transformations on them to produce target embeddings. By doing so, we mimic typologically dissimilar languages---see Figure~\ref{fig:example}.

\insertExample

As for the construction of simulation setups, we adjust three parameters: (a) the occurrence for a word $k$, $k \in \{5,10,\dots,100\}$---we use $20$ words in all setups; (b) the degree of language isomorphism $t \in \{1,\ldots,10\}$---which mimics different language pairs and (c) the distance $\epsilon$ between embeddings in train and test sets, $\epsilon \in \{0, 0.2,\dots,5\}$---which reflects different similarities between train and test domains. For the $i$-th word, in order to reflect word occurrence, we sample contextualized embeddings from a normal distribution $\mathcal{N}(\mu_i, \mtx{I})$ for train sets,  and from  $\mathcal{N}(\mu_i + \epsilon, \mtx{I})$ for validation and test sets, \wei{based on the insights from the visualized m-BERT space: different instances of a word appear to follow a normal distribution \citep{Cao:2020}.} $\mu_i$ denotes a mean vector sampled uniformly from $[-5, 5]$ for each component. For isomorphic languages ($t=1$), we transform source into target embeddings with a rotation matrix. For non-isomorphic languages ($t>1$), \wei{we alternate rotation with translation $t$ times, assuming the more often we alternate, the more dissimilar two languages become.}

\paragraph{Generalization to unseen words.}
\insertMain

Research showed that word frequency has a big impact on task performance for static embeddings \citep{Ruder:2019}. However, Figure~\ref{fig:main-simulation}~(a)+(b) show that, in the contextual case, task performance often does not correlate with word frequency but strongly correlates with domain similarities between train and test sets. \wei{On a side note, Figure~\ref{fig:main-simulation}(b) shows that Rotation correlates poorly with embedding distance in ``isomorphy'', but rather highly in ``non-isomorphy''. This is because isomorphic spaces can be perfectly aligned via Rotation, independent of the degree of embedding distance. For non-isomorphic spaces, the bigger the embedding distance is, the worse Rotation performs, which results in a high absolute correlation.}

Analyses by \citet{glavas-etal-2019-properly} showed that linear alignments are much better than non-linear counterparts on static embeddings. In the following, we contrast linear with non-linear alignments on contextualized embeddings, aiming to understand in which cases one is superior to another.

In isomorphic settings, Figure~\ref{fig:main-simulation}~(c) shows that linear alignments, Rotation and MUSE, clearly win in both supervised and unsupervised settings. This means a simple, linear transformation is sufficient to align vector spaces for isomorphic languages. We mark this as a sanity test, as languages mostly are non-isomorphic~\citep{sogaard:2018}.

In non-isomorphic settings, Figure~\ref{fig:main-simulation}~(d) shows that non-linear alignments, Real-NVP and GAN-Real-NVP, win by a large margin when train and test domains are similar. However, when train and test domains are dissimilar, linear alignments are indeed better. As such, non-linear alignments suffer from the issue of generalization.

Overall, we show that alignment on static and contextual embeddings  yield different conclusions on \emph{word frequency} and \emph{the superiority of linear over non-linear alignments}. By contrasting them, we hope to provide better understanding on each.

\insertDensity
\paragraph{Importance of bootstrapping procedure for unsupervised alignment.}
Figure~\ref{fig:density-fine-tuning}~(I) shows how well two languages are aligned when train and test domains are similar. In this context, Real-NVP and GAN-Real-NVP win, and the resulting vector spaces are better overlapped (aligned) than others.
This confirms the effectiveness of our density-based approaches.
However, we still see a big performance gap between supervised and unsupervised approaches, especially for Real-NVP (64.3) vs.~GAN-Real-NVP (51.3),  notwithstanding large overlap in subspaces and small differences in model architectures. This confirms that density matching alone is not sufficient. Figure~\ref{fig:density-fine-tuning}~(II) shows that after bootstrapping GAN-Real-NVP rivals 
Real-NVP. We also see similar results by contrasting Rotation with MUSE. Cross-correlation helps best in all cases, while graph-structure and Procrustes yield less consistent gains across approaches.

Overall, these results show that bootstrapping procedure plays a vital role in order for unsupervised alignments to rival supervised counterparts.

\subsection{Experiments on Real Data}
\label{sec:real-data}
XTREME~\citep{XTREME:2020} has recently become popular for evaluating multilingual representations. However, it does not address word-level alignment as CBLI and BLI do, but rather focus on how multilingual representations impact cross-lingual systems. 
In this work, we evaluate both internal and external strengths of alignment, i.e., the internal alignment results on CBLI,
and the impact of alignment on external tasks: (i) Align, RFEval and Tatoeba that require no supervised classifiers, and (ii) XNLI that requires a supervised classifier. We outline these tasks in the following:

\begin{itemize}
    \item CBLI is the contextualized extension of BLI.
    Both contain a bilingual lexicon per language pair, but CBLI marks each occurrence of a word as an entry in lexicon. For each language pair, we extract 10k word translations from parallel sentences using FastAlign ~\citep{dyer:2013}.  
    We report Precision@1. Note that we provide two complementary CBLI data: one is gold but simulated, while the other is real  but contains noises.
    
    \item Alignment (Align) is a bilingual word retrieval task.
    Each language pair contains gold standard 2.5k word translations annotated by human experts. We use SimAlign~\citep{simalign:2020} to retrieve word translations from parallel sentences based on contextualized word embeddings. We report F-score that combines precision and recall.
    
    \item Reference-free evaluation (RFEval) measures the Pearson correlation between human and automatic judgments of translation quality. 
    We use XMover~\citep{zhao:2019, zhao:2020-refeval} to yield automatic judgment, which compares system translation with source sentence based on contextualized word embeddings. We exclude the target-side language model from XMover. Each language pair contains 3k source sentences.
    
    \item Tatoeba is a bilingual sentence retrieval task taken from XTREME. Each language pair contains 1k sentence translations. Given a source sentence, we retrieve the nearest translation from a pool of candidates based on cosine similarities between sentence embeddings. We report Precision@1.
    
    \item XNLI is a cross-lingual transfer task taken from XTREME, which aims to infer the relationship between a sentence pair of premise and hypothesis. Often, XNLI is evaluated in a zero-shot transfer setup, which measures the transfer ability from source to target languages, with cross-lingual systems trained on source language only. We report accuracy.

\end{itemize}

Tatoeba, CBLI and RFEval consist of six languages: German, Czech, Latvian, Finnish, Russian and Turkish, paired to English. In this work, we train alignments for these language pairs. Align considers two language pairs: German/Czech-to-English, and XNLI considers three: English-to-German/Russian/Turkish, as the other languages are not available. We consider two choices of multilingual representations: m-BERT and XLM-R. 

\paragraph{Setup.} To contrast supervised with unsupervised approaches, we consider two data scenarios: (i) limited parallel data and (ii) no parallel data. In case (i), we sample 20k (compared to ca. 250k often used in previous studies) parallel sentences from News-Commentary~\citep{news-commentary:2012} for Russian/Turkish-to-English, and from EuroParl~\citep{koehn:2005} for other languages.  We use FastAlign to induce word translations from parallel sentences. Building upon these translations, we construct a permutation matrix $\mtx{P}$ as cross-lingual supervision. In case (ii), we unpair the word translations obtained from (i) by removing the use for the permutation matrix. As such, we compare supervised and unsupervised approaches under similar conditions,  \emph{viz.}, with similar scale of data.

\paragraph{How to select the best model.}
We compare two choices of 
model selection: 
(i) CBLI as validation data and (ii) our validation criteria.

Figure~\ref{fig:validation-criterion}~(I) shows that  the results on CBLI and on other tasks correlate poorly (even negatively) in both supervised and unsupervised settings. This means validation data is inappropriate for guiding both supervised and unsupervised training. 

\insertValidationCriterion
\insertTaskCorrelation

Table~\ref{tab:task-correlation} reports correlation statistics across approaches, showing that the model performances across tasks exhibit negative correlations in about $30\%$ setups, i.e., the better the alignment performs in one task, the worse it performs in the other. This means that (i) guiding contextualized alignment with validation criteria is challenging, and (ii) the evaluation of validation criteria should consider the performances in all tasks. 
As a running example, we evaluate the two validation criteria and validation data (CBLI) treated as criterion, based on Real-NVP. Figure~\ref{fig:validation-criterion}~(II) shows that both validation criteria correlate much better than validation data with the model performances across tasks. Further, \emph{semantic criterion} wins by a large margin (0.86 versus 0.44 and -0.12). This means \emph{semantic criterion} is the best option for guiding Real-NVP. We also see similar results on other approaches. Thus, we adopt \emph{semantic criterion} to perform model selection in all setups. 

Overall, these results show that (i) not only are validation criteria important for guiding unsupervised training, but also for guiding supervised training, and (ii) the evaluation of validation criteria should be based on the model performances in all tasks.

\subsection{Results on Real Data}

Table~\ref{tab:main-task-performances} contrasts unsupervised with supervised approaches. For ease of reading, we provide the average results across languages, and break them down into individual languages in Table~\ref{tab:full-results} (appendix).

\insertTableTaskPerformance

\paragraph{Supervised settings.} 
FCNN and Real-NVP training for 5 epochs are worse than those training with \emph{semantic criterion} in almost all setups. This demonstrates the importance of validation criteria. We find that, though Joint-Align training with 100k parallel data wins on internal CBLI, it is worse than Real-NVP training with (i) merely 20k data and (ii) \emph{semantic criterion} in the external tasks. This means  Joint-Align overfits CBLI. We see similar results by contrasting VecMap with GAN-Real-NVP. 

We see that the gains from all alignments on Align are much smaller than on others. This might be because SimAlign (used to induce word alignments) or the dataset cannot recognize the improved contextualized embeddings after alignment.

Real-NVP seems the strongest approach, which helps considerably for both m-BERT and XLM-R and surpasses recent InfoXLM by a large margin. InfoXLM training with much larger parallel data for 150k training steps cannot show advantages in the absence of validation criteria. \wei{Note that we do not apply validation criteria to Joint-Align and InfoXLM for further improvements, as these resource-intensive approaches have not been designed for low-resource languages, e.g., that the improvements by Joint-Align appear to vanish in the setup of limited parallel data~\cite{Cao:2020}.}

\paragraph{Unsupervised settings.} 
Validation criteria are crucial: MUSE and GAN-Real-NVP with \emph{semantic criterion} largely outperform those training for 5 epochs.

Much unlike the results in simulation, we see small performance gaps between supervised and unsupervised approaches, such as the gap between Real-NVP and GAN-Real-NVP (2 points vs 13 points in simulation). Thus, it is not surprising that the gains from the bootstrapping procedure are small in these tasks. Overall, we see \emph{cross-correlation} is better than \emph{graph structure} on MUSE and GAN-Real-NVP. 
The results for Procrustes are similar as in simulation---it improves MUSE but harms GAN-Real-NVP. GAN-Real-NVP training with \emph{semantic criterion} and \emph{cross-correlation} always wins, rivaling the best supervised approach Real-NVP.

Overall, these supervised and unsupervised results show that (i) validation criterion plays an essential role; (ii) density-based approaches targeting density matching and density modeling are effective in both supervised and unsupervised settings, and (iii) after bootstrapping unsupervised approaches are able to rival supervised counterparts. 

\paragraph{Downstream Task.} 

\begin{table}[!ht]
\centering
\footnotesize
\begin{tabular}{@{} l | c c c @{} }
\toprule
& \multicolumn{3}{c}{m-BERT} \\
Alignments  & DE & RU & TR \\
\midrule
Original & 70.3 & 68.2 & 60.0 \\
\midrule
Rotation-20k & 70.6 & 68.1 & 60.5\\
FCNN-20k (semantic criterion) & 70.8 & 68.1 & 59.9 \\
Real-NVP-20k (semantic criterion) & \textbf{73.6} & \textbf{72.7} & \textbf{62.9} \\
Joint-Align-100k (3 epochs) & \textbf{72.9} & \textbf{72.1} & \textbf{62.4} \\
\midrule
GAN-R-NVP-20k (semantic criterion) & \textbf{73.5} & \textbf{72.2} & 61.5 \\
+ Cross-Correlation  & \textbf{73.4} & \textbf{72.1} & 61.3 \\
+ Graph Structure & \textbf{73.4} & \textbf{72.3} & 61.2\\
+ Procrustes & 70.4 & 68.3 & 60.2 \\
\bottomrule
\end{tabular}
\caption{Results on XNLI in a zero-shot cross-lingual setup from English to German/Russian/Turkish. After rectifying m-BERT with alignments we finetune m-BERT (coupled with a supervised classifier) on XNLI English train data.
We restrict the evaluation to 3 languages, as the other languages on XNLI are not covered in our alignments. 
\label{tab:results-xnli}}
\end{table}
Table \ref{tab:results-xnli} shows the impacts of our approaches coupled with a supervised 
classifier to perform zero-shot text classification on the downstream task XNLI.
While Rotation and FCNN yield better results in Table \ref{tab:main-task-performances}, their impacts 
vanish on XNLI. This might be because 
(i) rectifying m-BERT with 20k parallel data is not adequate to reflect improvements on downstream tasks, or (ii) alignment results may be orthogonal to downstream zero-shot performance. However, 
Real-NVP and GAN-Real-NVP trained on the same scale of data with Rotation and FCNN exhibit strong impacts on XNLI, 
on par with Joint-Align trained with 100k parallel data. 
Thus, 20k data is sufficient for our approaches to yield 
improvements on XNLI. 

\section{Conclusion}

Given resource and typology disparities across languages, multilingual representations exhibit unequal capabilities between languages. Research showed that contextualized alignments overcome this challenge by producing language-agnostic representations. However, these techniques demand large parallel data, and thus cannot address the data scarcity issue in low-resource languages. 
Our contributions in this work are manifold. We start by introducing supervised and unsupervised density-based approaches, Real-NVP and GAN-Real-NVP, both dissecting the alignment of multilingual subspaces into density matching and density modeling in order to sufficiently leverage data. Second, we investigate the usefulness of validation criteria for guiding the training process of our approaches. 
Further, we present a bootstrapping procedure to enhance our unsupervised approach, which is theoretically grounded for promoting equality of multilingual subspaces. We demonstrated the effectiveness of our alignments in the scenarios of limited and no parallel data. With 20k parallel data we provided, our supervised approach mostly outperforms Joint-Align and InfoXLM trained on much larger parallel data. Next, we showed that validation criteria are imperative for guiding both supervised and unsupervised training. 
Finally, 
we demonstrated that parallel data could be removed without the loss of model performances after integrating our unsupervised approach in the bootstrapping procedure.

\section{Broader Impact}
\label{sec:broader-impact}
As a class of domain adaptation techniques, density-based approaches have been shown useful in a range of cross-domain applications, such as image-captioning~\citep{mahajan2020latent}, image-to-image translation~\citep{alignflow:2020}, alignment on static embeddings~\citep{zhou-etal-2019-density} and machine translation~\citep{setiawan-etal-2020-variational}. In this work, we showed that (i) density-based approaches could overfit validation data in the absence of validation criteria, and are weak in generalization (see \S\ref{sec:simulation}), but (ii) bootstrapping procedures can improve these density-based approaches. While our analyses are limited in scope to contextualized alignment as the only cross-domain application, we hope that our results fuel future research towards effective domain adaptation techniques in other applications.

\section*{Acknowledgments}
We thank the anonymous reviewers for their thoughtful comments that greatly
improved the final texts. We also
thank Dan Liu for her early experiments in the word alignment task.
This work has been supported by the German Research Foundation as part of the Research Training
Group Adaptive Preparation of Information from Heterogeneous Sources (AIPHES) at the Technische
Universit\"at Darmstadt under grant No. GRK 1994/1 and the Klaus Tschira Foundation, Heidelberg, Germany. 

\bibliography{acml22, custom}

\begin{thebibliography}{35}
\providecommand{\natexlab}[1]{#1}
\providecommand{\url}[1]{\texttt{#1}}
\expandafter\ifx\csname urlstyle\endcsname\relax
  \providecommand{\doi}[1]{doi: #1}\else
  \providecommand{\doi}{doi: \begingroup \urlstyle{rm}\Url}\fi

\bibitem[Aldarmaki and Diab(2019)]{aldarmaki:2019}
Hanan Aldarmaki and Mona Diab.
\newblock Context-aware cross-lingual mapping.
\newblock In \emph{NAACL}, 2019.

\bibitem[Alqahtani et~al.(2021)Alqahtani, Lalwani, Zhang, Romeo, and
  Mansour]{ot:2021}
Sawsan Alqahtani, Garima Lalwani, Yi~Zhang, Salvatore Romeo, and Saab Mansour.
\newblock Using optimal transport as alignment objective for fine-tuning
  multilingual contextualized embeddings.
\newblock In \emph{EMNLP}, 2021.

\bibitem[Arjovsky et~al.(2017)Arjovsky, Chintala, and Bottou]{wgan:2017}
Martin Arjovsky, Soumith Chintala, and L{\'e}on Bottou.
\newblock {W}asserstein generative adversarial networks.
\newblock In \emph{ICML}, 2017.

\bibitem[Artetxe et~al.(2017)Artetxe, Labaka, and
  Agirre]{artetxe-etal-2017-learning}
Mikel Artetxe, Gorka Labaka, and Eneko Agirre.
\newblock Learning bilingual word embeddings with (almost) no bilingual data.
\newblock In \emph{ACL}, Vancouver, Canada, July 2017.

\bibitem[Artetxe et~al.(2018)Artetxe, Labaka, and Agirre]{vecmap:2018}
Mikel Artetxe, Gorka Labaka, and Eneko Agirre.
\newblock A robust self-learning method for fully unsupervised cross-lingual
  mappings of word embeddings.
\newblock In \emph{ACL}, Melbourne, Australia, July 2018.

\bibitem[Artetxe et~al.(2020)Artetxe, Ruder, Yogatama, Labaka, and
  Agirre]{artetxe-etal-2020-call}
Mikel Artetxe, Sebastian Ruder, Dani Yogatama, Gorka Labaka, and Eneko Agirre.
\newblock A call for more rigor in unsupervised cross-lingual learning.
\newblock In \emph{ACL}, pages 7375--7388, Online, July 2020.

\bibitem[Cao et~al.(2020)Cao, Kitaev, and Klein]{Cao:2020}
Steven Cao, Nikita Kitaev, and Dan Klein.
\newblock Multilingual alignment of contextual word representations.
\newblock In \emph{ICLR}, 2020.

\bibitem[Chi et~al.(2021)Chi, Dong, Wei, Yang, Singhal, Wang, Song, Mao, Huang,
  and Zhou]{infoxlm}
Zewen Chi, Li~Dong, Furu Wei, Nan Yang, Saksham Singhal, Wenhui Wang, Xia Song,
  Xian-Ling Mao, Heyan Huang, and Ming Zhou.
\newblock {I}nfo{XLM}: An information-theoretic framework for cross-lingual
  language model pre-training.
\newblock In \emph{NAACL}, 2021.

\bibitem[Conneau et~al.(2020)Conneau, Khandelwal, Goyal, Chaudhary, Wenzek,
  Guzm{\'a}n, Grave, Ott, Zettlemoyer, and
  Stoyanov]{conneau-etal-2020-unsupervised}
Alexis Conneau, Kartikay Khandelwal, Naman Goyal, Vishrav Chaudhary, Guillaume
  Wenzek, Francisco Guzm{\'a}n, Edouard Grave, Myle Ott, Luke Zettlemoyer, and
  Veselin Stoyanov.
\newblock Unsupervised cross-lingual representation learning at scale.
\newblock In \emph{ACL}, Online, July 2020.

\bibitem[Czarnowska et~al.(2019)Czarnowska, Ruder, Grave, Cotterell, and
  Copestake]{Ruder:2019}
Paula Czarnowska, Sebastian Ruder, Edouard Grave, Ryan Cotterell, and Ann
  Copestake.
\newblock Don{'}t forget the long tail! a comprehensive analysis of
  morphological generalization in bilingual lexicon induction.
\newblock In \emph{EMNLP}, 2019.

\bibitem[Devlin et~al.(2019)Devlin, Chang, Lee, and
  Toutanova]{devlin-etal-2019-bert}
Jacob Devlin, Ming-Wei Chang, Kenton Lee, and Kristina Toutanova.
\newblock {BERT}: Pre-training of deep bidirectional transformers for language
  understanding.
\newblock In \emph{NAACL}, 2019.

\bibitem[Dinh et~al.(2017)Dinh, Sohl{-}Dickstein, and Bengio]{Dinh:2017}
Laurent Dinh, Jascha Sohl{-}Dickstein, and Samy Bengio.
\newblock Density estimation using real {NVP}.
\newblock In \emph{ICLR}, 2017.

\bibitem[Dubossarsky et~al.(2020)Dubossarsky, Vuli{\'c}, Reichart, and
  Korhonen]{dubossarsky:2020}
Haim Dubossarsky, Ivan Vuli{\'c}, Roi Reichart, and Anna Korhonen.
\newblock The secret is in the spectra: Predicting cross-lingual task
  performance with spectral similarity measures.
\newblock In \emph{EMNLP}, 2020.

\bibitem[Dyer et~al.(2013)Dyer, Chahuneau, and Smith]{dyer:2013}
Chris Dyer, Victor Chahuneau, and Noah~A. Smith.
\newblock A simple, fast, and effective reparameterization of {IBM} model 2.
\newblock In \emph{NAACL}, 2013.

\bibitem[Faruqui and Dyer(2014)]{faruqui:2014}
Manaal Faruqui and Chris Dyer.
\newblock Improving vector space word representations using multilingual
  correlation.
\newblock In \emph{EACL}, 2014.

\bibitem[Glava{\v{s}} et~al.(2019)Glava{\v{s}}, Litschko, Ruder, and
  Vuli{\'c}]{glavas-etal-2019-properly}
Goran Glava{\v{s}}, Robert Litschko, Sebastian Ruder, and Ivan Vuli{\'c}.
\newblock How to (properly) evaluate cross-lingual word embeddings: On strong
  baselines, comparative analyses, and some misconceptions.
\newblock In \emph{ACL}, Florence, Italy, July 2019.

\bibitem[Goodfellow et~al.(2014)Goodfellow, Pouget-Abadie, Mirza, Xu,
  Warde-Farley, Ozair, Courville, and Bengio]{goodfellow:2014}
Ian Goodfellow, Jean Pouget-Abadie, Mehdi Mirza, Bing Xu, David Warde-Farley,
  Sherjil Ozair, Aaron Courville, and Yoshua Bengio.
\newblock Generative adversarial nets.
\newblock In \emph{NeurIPS}, 2014.

\bibitem[Grover et~al.(2020)Grover, Chute, Shu, Cao, and Ermon]{alignflow:2020}
Aditya Grover, Christopher Chute, Rui Shu, Zhangjie Cao, and Stefano Ermon.
\newblock Alignflow: Cycle consistent learning from multiple domains via
  normalizing flows.
\newblock In \emph{AAAI}, 2020.

\bibitem[Hu et~al.(2020)Hu, Ruder, Siddhant, Neubig, Firat, and
  Johnson]{XTREME:2020}
Junjie Hu, Sebastian Ruder, Aditya Siddhant, Graham Neubig, Orhan Firat, and
  Melvin Johnson.
\newblock {XTREME}: A massively multilingual multi-task benchmark for
  evaluating cross-lingual generalisation.
\newblock In \emph{ICML}, 2020.

\bibitem[Jalili~Sabet et~al.(2020)Jalili~Sabet, Dufter, Yvon, and
  Sch{\"u}tze]{simalign:2020}
Masoud Jalili~Sabet, Philipp Dufter, Fran{\c{c}}ois Yvon, and Hinrich
  Sch{\"u}tze.
\newblock {S}im{A}lign: High quality word alignments without parallel training
  data using static and contextualized embeddings.
\newblock In \emph{Findings of EMNLP}, 2020.

\bibitem[Koehn(2005)]{koehn:2005}
Philipp Koehn.
\newblock Europarl: A parallel corpus for statistical machine translation.
\newblock In \emph{MT summit}, volume~5. Citeseer, 2005.

\bibitem[Lample et~al.(2018)Lample, Conneau, Ranzato, Denoyer, and
  Jégou]{lample:2018}
Guillaume Lample, Alexis Conneau, Marc'Aurelio Ranzato, Ludovic Denoyer, and
  Hervé Jégou.
\newblock Word translation without parallel data.
\newblock In \emph{ICLR}, 2018.

\bibitem[Libovick{\'y} et~al.(2020)Libovick{\'y}, Rosa, and
  Fraser]{libovicky-etal-2020-language}
Jind{\v{r}}ich Libovick{\'y}, Rudolf Rosa, and Alexander Fraser.
\newblock On the language neutrality of pre-trained multilingual
  representations.
\newblock In \emph{Findings of EMNLP}, Online, November 2020.

\bibitem[Mahajan et~al.(2020)Mahajan, Gurevych, and Roth]{mahajan2020latent}
Shweta Mahajan, Iryna Gurevych, and Stefan Roth.
\newblock Latent normalizing flows for many-to-many cross-domain mappings.
\newblock In \emph{ICLR}, 2020.

\bibitem[Mengzhou et~al.(2021)Mengzhou, Zheng, Mukherjee, Shokouhi, Newbig, and
  Awadallah]{xia2021metaxl}
Xia Mengzhou, Guoqing Zheng, Subhabrata Mukherjee, Milad Shokouhi, Graham
  Newbig, and Ahmed~Hassan Awadallah.
\newblock Metaxl: Meta representation transformation for low-resource
  cross-lingual learning.
\newblock 2021.

\bibitem[Pires et~al.(2019)Pires, Schlinger, and Garrette]{pires:2019}
Telmo Pires, Eva Schlinger, and Dan Garrette.
\newblock How multilingual is multilingual {BERT}?
\newblock In \emph{ACL}, Florence, Italy, July 2019.

\bibitem[Rezende and Mohamed(2015)]{nf:2015}
Danilo Rezende and Shakir Mohamed.
\newblock Variational inference with normalizing flows.
\newblock In \emph{ICML}, 2015.

\bibitem[Setiawan et~al.(2020)Setiawan, Sperber, Nallasamy, and
  Paulik]{setiawan-etal-2020-variational}
Hendra Setiawan, Matthias Sperber, Udhyakumar Nallasamy, and Matthias Paulik.
\newblock Variational neural machine translation with normalizing flows.
\newblock In \emph{ACL}, July 2020.

\bibitem[S{\o}gaard et~al.(2018)S{\o}gaard, Ruder, and Vuli{\'c}]{sogaard:2018}
Anders S{\o}gaard, Sebastian Ruder, and Ivan Vuli{\'c}.
\newblock On the limitations of unsupervised bilingual dictionary induction.
\newblock In \emph{ACL}, 2018.

\bibitem[Tiedemann(2012)]{news-commentary:2012}
Jörg Tiedemann.
\newblock Parallel data, tools and interfaces in opus.
\newblock In \emph{LREC}, 2012.

\bibitem[Wu and Dredze(2020)]{wu-dredze-2020-explicit}
Shijie Wu and Mark Dredze.
\newblock Do explicit alignments robustly improve multilingual encoders?
\newblock In \emph{EMNLP}, Online, November 2020.

\bibitem[Zhao et~al.(2019)Zhao, Peyrard, Liu, Gao, Meyer, and Eger]{zhao:2019}
Wei Zhao, Maxime Peyrard, Fei Liu, Yang Gao, Christian~M. Meyer, and Steffen
  Eger.
\newblock {M}over{S}core: Text generation evaluating with contextualized
  embeddings and earth mover distance.
\newblock In \emph{EMNLP}, Hong Kong, China, November 2019.

\bibitem[Zhao et~al.(2020)Zhao, Glava{\v{s}}, Peyrard, Gao, West, and
  Eger]{zhao:2020-refeval}
Wei Zhao, Goran Glava{\v{s}}, Maxime Peyrard, Yang Gao, Robert West, and
  Steffen Eger.
\newblock On the limitations of cross-lingual encoders as exposed by
  reference-free machine translation evaluation.
\newblock In \emph{ACL}, 2020.

\bibitem[Zhao et~al.(2021)Zhao, Eger, Bjerva, and Augenstein]{zhao:2020}
Wei Zhao, Steffen Eger, Johannes Bjerva, and Isabelle Augenstein.
\newblock Inducing language-agnostic multilingual representations.
\newblock In \emph{*SEM}, Online, August 2021.

\bibitem[Zhou et~al.(2019)Zhou, Ma, Wang, and Neubig]{zhou-etal-2019-density}
Chunting Zhou, Xuezhe Ma, Di~Wang, and Graham Neubig.
\newblock Density matching for bilingual word embedding.
\newblock In \emph{NAACL}, Minneapolis, Minnesota, June 2019.

\end{thebibliography}

\appendix
\clearpage

\begin{sidewaystable*}
    \centering
    \scriptsize
    \setlength{\tabcolsep}{3pt}
\begin{tabular}{l| ccccc ccccc ccccccc}
    \toprule
            Lang & Original & Rotation & FCNN & NVP & GBDD & Joint-Align & NORM & MUSE & MUSE$+$C & MUSE$+$G & MUSE$+$P & GNVP & GNVP$+$C & GNVP$+$G & GNVP$+$P & infoXLM & VecMap\\
            \midrule
                        
            \multicolumn{16}{l}{\textit{Reference-free Evaluation (m-BERT)}}\\
            cs-en &23.6& 32.5& 36.2& 30.2& 24.4& 32.5& 22.5& 24.8& 29.7& 26.7& 30.4& 24.9& 25.3& 25.2& 30.5 &- & 25.0\\
            de-en &29.8& 31.6& 33.1& 34.2& 30.0& 39.9& 31.4& 29.9& 31.3& 31.1& 31.2& 32.8& 33.0& 32.8& 31.3&- & 31.8\\
            fi-en &30.9& 48.5& 51.3& 51.8& 32.5& 48.6& 32.3& 37.9& 44.0& 40.4& 46.0& 37.2& 39.5& 37.9& 46.1&- & 39.1\\
            lv-en &23.0& 44.0& 50.2& 52.1& 25.7& 46.1& 31.3& 33.4& 35.9& 33.5& 40.9& 38.0& 40.7& 38.8& 41.0&-& 24.9\\
            ru-en &19.4& 23.4& 30.7& 33.4& 21.4& 31.7& 23.9& 19.8& 22.0& 21.3& 22.9& 26.5& 27.0& 26.3& 22.9&-&18.7\\
            tr-en &36.7& 52.4& 54.8& 52.2& 38.6& 48.6& 39.1& 41.6& 48.6& 46.3& 49.7& 44.0& 46.5& 44.9& 49.8&-&44.5\\
        \midrule
            \multicolumn{16}{l}{\textit{Reference-free Evaluation (XLM-R)}}\\
            cs-en &20.6& 26.3& 26.1& 32.9& 22.4& - & 28.0& 20.4& 24.2& 22.3& 24.6& 29.2& 30.1& 29.1& 25.1 & 26.6\\
            de-en &25.1& 27.4& 28.2& 40.5& 25.4& - &39.5& 25.2& 26.5& 26.0& 26.6& 40.4& 40.4& 40.5& 27.2 & 37.3\\
            fi-en &26.6& 39.9& 44.6& 47.7& 26.0& - &41.6& 29.0& 35.7& 32.3& 37.8& 44.3& 47.2& 44.5& 38.2 & 36.1\\
            lv-en &28.3& 42.3& 48.6& 51.2& 29.6& - &44.1& 30.8& 37.9& 32.7& 40.7& 47.3& 48.7& 47.4& 40.9 & 38.3\\
            ru-en &21.4& 25.6& 27.7& 38.0& 21.5& - &34.3& 21.5& 23.3& 22.4& 24.6& 35.9& 36.3& 36.1& 25.1 & 34.9\\
            tr-en &36.5& 46.5& 44.8& 54.7& 37.9& - &49.6& 38.4& 44.7& 41.2& 45.9& 53.2& 52.9& 51.7& 46.2 & 52.6\\
        \midrule
            \multicolumn{16}{l}{\textit{Tatoeba (m-BERT)}}\\
            cs-en &47.5& 53.1& 61.4& 61.3& 48.3& 51.1& 60.6& 48.2& 50.8& 48.7& 51.6& 60.6& 61.3& 60.7& 51.5 &- & 47.5\\
            de-en &76.6& 79.5& 84.5& 87.3& 78.6& 89.6& 86.0& 77.7& 79.3& 78.7& 78.9& 86.2& 86.5& 86.0& 78.8 &- & 81.2\\
            fi-en &41.5& 50.6& 57.2& 54.3& 46.4& 70.4& 53.5& 44.7& 46.5& 45.2& 47.8& 53.9& 56.1& 54.6& 47.5 &- & 53.5\\
            lv-en &31.7& 39.8& 47.4& 51.2& 35.0& 31.1& 44.0& 35.3& 35.0& 33.8& 37.1& 44.1& 44.7& 43.8& 37.0 &- & 33.5\\
            ru-en &63.0& 65.7& 70.3& 72.6& 64.8& 74.2& 75.4& 63.5& 64.9& 64.7& 66.3& 75.1& 76.0& 75.1& 66.3 &-& 66.5\\
            tr-en &35.8& 43.0& 46.3& 50.5& 40.6& 38.4& 48.2& 39.1& 40.9& 38.8& 43.1& 47.3& 49.3& 48.1& 42.6 &- & 47.6\\            
        \midrule
            \multicolumn{16}{l}{\textit{Tatoeba (XLM-R)}}\\
            cs-en &49.8& 57.5& 74.6& 72.9& 59.2& - &72.5& 53.8& 58.9& 60.6& 59.1& 72.8& 74.0& 73.3& 57.6 & 69.6\\
            de-en &89.0& 91.0& 94.5& 94.9& 90.1& - &95.2& 88.7& 91.0& 89.6& 90.8& 94.6& 95.2& 94.8& 90.5 & 95.5\\
            fi-en &63.8& 68.6& 81.3& 80.8& 69.3& - &79.6& 65.6& 69.8& 67.2& 68.8& 80.6& 81.2& 80.6& 67.7 & 68.4\\
            lv-en &52.2& 61.3& 72.6& 72.4& 58.9& - &71.9& 56.4& 63.1& 59.2& 61.0& 71.5& 71.5& 71.7& 59.7 & 51.2\\
            ru-en &70.2& 71.5& 83.1& 83.9& 73.2& - &83.9& 71.2& 70.6& 70.2& 72.2& 83.4& 84.2& 84.0& 70.9 & 85.7\\
            tr-en &55.4& 61.7& 75.1& 75.6& 62.4& - &75.4& 58.6& 62.0& 63.2& 61.0& 75.1& 76.9& 75.7& 61.3 & 86.1\\
        \midrule
            \multicolumn{16}{l}{\textit{Contextual BLI (m-BERT)}}\\
            cs-en&47.3&52.6&49.2&52.3&47.7&57.0&50.1&47.7&50.0&48.4&50.8&51.8&52.0&51.5&51.0 &- &55.1\\
            de-en&61.0&64.6&59.0&63.8&60.9&79.4&63.3&61.8&61.8&62.1&62.5&64.8&64.8&64.7&63.0 &-&71.7\\
            fi-en&36.8&48.9&45.8&45.8&37.9&74.2&43.0&40.7&40.6&40.7&43.4&45.0&47.6&44.7&44.6 &-&56.7\\
            lv-en&42.7&51.8&49.5&50.6&43.5&49.8&47.9&45.5&44.5&45.9&48.7&49.9&50.5&49.7&49.5 &-&53.9\\
            ru-en&66.5&70.9&67.6&71.4&67.0&67.4&69.2&67.2&67.5&67.8&69.5&70.5&71.0&70.6&69.9 &-&76.9\\
            tr-en&51.1&61.9&60.7&61.8&51.5&60.2&55.8&52.0&52.8&53.7&56.4&60.2&62.1&58.7&58.3 &-&70.4\\
        \midrule
            \multicolumn{16}{l}{\textit{Contextual BLI (XLM-R)}}\\     
            cs-en &40.2&45.1&43.6&51.1&39.3&- &48.2&41.6&42.0&41.5&42.3&49.8&50.0&49.6&43.3&50.7\\
            de-en &54.1&57.9&53.2&64.6&54.3&- &62.7&55.2&55.1&55.3&55.4&65.1&65.3&64.9&56.3&67.3\\
            fi-en &44.2&51.0&46.1&60.9&44.0&- &58.1&46.8&45.7&47.2&46.6&60.2&60.5&59.9&48.3&57.1\\
            lv-en &44.2&50.0&48.3&57.6&44.2&- &54.8&46.4&45.7&47.3&46.5&56.9&57&56.7&47.8&53.5\\
            ru-en &60.9&64.4&60.3&71.7&60.7&- &70.3&61.4&61.8&61.2&62.2&71.5&71.5&71.4&63.0&72.4\\
            tr-en &47.9&53.6&53.8&63.9&48.0&- &60.1&48.6&48.5&48.6&49.2&62.5&63.1&62.5&50.9&64.0\\
            \midrule
            \multicolumn{16}{l}{\textit{Word Alignment (m-BERT)}}\\ 
            cs-en &44.0& 45.0& 43.4& 45.7& 44.1& 44.2 &44.6& 44.0& 45.0& 44.5& 45.0& 44.6& 45.1& 44.9& 45.1&- & 45.0\\
            de-en &79.1& 80.6& 79.6& 79.5& 79.3& 80.2 &80.5& 79.3& 80.2& 79.9& 80.4& 80.0& 80.3& 80.2& 80.5&-& 80.0\\
            \multicolumn{16}{l}{\textit{Word Alignment (XLM-R)}}\\
            cs-en &41.1& 42.1& 41.8& 45.2& 40.9& - &44.4& 41.2& 42.1& 41.8& 42.1& 44.6& 44.9& 44.8& 42.0&43.3\\
            de-en &78.5& 79.6& 77.9& 80.5& 78.7& - &80.8& 78.8& 78.9& 79.2& 79.1& 81.0& 81.4& 81.2& 79.0&82.6\\
    \midrule
\end{tabular}
\caption{Full results of baselines and our approaches. Real-NVP and GAN-Real-NVP are denoted by NVP and GNVP.  [Method]$+$P/C/G denotes the integration in Procrustes refinement (P), in our bootstrapping procedure constrained with cross-correlation (C) and with graph structure (G).}
\label{tab:full-results}
\end{sidewaystable*}

\end{document}